\documentclass[sigconf]{aamas} 

\usepackage{balance} 
\usepackage{amsmath,amsfonts}
\usepackage{algorithm}
\usepackage[noend]{algpseudocode}
\usepackage{graphicx}
\usepackage{textcomp}
\usepackage{tikz}
\usetikzlibrary{arrows.meta}
\usetikzlibrary{decorations.markings}
\usepackage{subcaption}
\usepackage{epsfig}
\usepackage{booktabs}

\usepackage[shortlabels]{enumitem}
\usepackage{multirow}
\newtheorem{lemma}{Lemma}
\newtheorem{theorem}{Theorem}

\usepackage{amsthm} 
\usepackage{cancel}
\usepackage{wrapfig,lipsum,booktabs}
\usepackage{natbib}

\def\ls#1{{\color{black}#1}}

\setcopyright{ifaamas}
\acmConference[AAMAS '25]{Proc.\@ of the 24th International Conference
on Autonomous Agents and Multiagent Systems (AAMAS 2025)}{May 19 -- 23, 2025}
{Detroit, Michigan, USA}{A.~El~Fallah~Seghrouchni, Y.~Vorobeychik, S.~Das, A.~Nowe (eds.)}
\copyrightyear{2025}
\acmYear{2025}
\acmDOI{}
\acmPrice{}
\acmISBN{}

\acmSubmissionID{842}

\title[Experience-replay Innovative Dynamics]{Experience-replay Innovative Dynamics}

\author{Tuo Zhang}
\affiliation{
  \institution{University of Birmingham}
  \city{Birmingham}
  \country{United Kingdom}}
\email{txz257@student.bham.ac.uk}

\author{Leonardo Stella}
\affiliation{
  \institution{University of Birmingham}
  \city{Birmingham}
  \country{United Kingdom}}
\email{l.stella@bham.ac.uk}

\author{Julian Barreiro Gomez}
\affiliation{
  \institution{Khalifa University}
  \city{Abu Dhabi}
  \country{United Arab Emirates}}
\email{julian.barreirogomez@ku.ac.ae}

\begin{abstract}

Despite its groundbreaking success, multi-agent reinforcement learning (MARL) still suffers from instability and nonstationarity. Replicator dynamics, the most well-known model from evolutionary game theory (EGT), provide a theoretical framework for the convergence of the trajectories to Nash equilibria and, as a result, have been used to ensure formal guarantees for MARL algorithms in stable game settings. However, they exhibit the opposite behavior in other settings, which poses the problem of finding alternatives to ensure convergence. In contrast, innovative dynamics, such as the Brown-von Neumann-Nash (BNN) or Smith, result in periodic trajectories with the potential to approximate Nash equilibria. Yet, no MARL algorithms based on these dynamics have been proposed. In response to this challenge, we develop a novel experience replay-based MARL algorithm that incorporates revision protocols as tunable hyperparameters. We demonstrate, by appropriately adjusting the revision protocols, that the behavior of our algorithm mirrors the trajectories resulting from these dynamics. Importantly, our contribution provides a framework capable of extending the theoretical guarantees of MARL algorithms beyond replicator dynamics. Finally, we corroborate our theoretical findings with empirical results.

\end{abstract}


\newcommand{\BibTeX}{\rm B\kern-.05em{\sc i\kern-.025em b}\kern-.08em\TeX}

\begin{document}


\pagestyle{fancy}
\fancyhead{}


\maketitle 


\section{Introduction}
Multi-agent reinforcement learning (MARL) has demonstrated significant success across various domains, including games such as Go and real-time strategy games, robotic control, cyber-physical systems, finance, and sensor networks, where numerous agents interact within complex environments \cite{silver2017mastering, vinyals2019grandmaster, lillicrap2015continuous, adler2002cooperative, lee2007multiagent, cortes2004coverage}. Despite its success, MARL faces several challenges. One of the main open questions is to provide theoretical guarantees of convergence and optimality under general conditions. As such, a challenge faced by MARL algorithms is the nonstationarity induced by the change in policy of the agents while they learn concurrently. Indeed, the rewards that each agent receives are determined not only through its policy, but also through the joint policy of the other agents~\cite{zhang2021multi, yang2020overview}.

Evolutionary game theory (EGT) studies the evolution of strategic interactions in a population of decision-makers, where the fitness of a strategy increases based on the success of that strategy in a given environment~\cite{smith1973logic, smith1974theory, weibull1997evolutionary}. EGT has played a critical role in the analysis and evaluation of MARL algorithms in complex multi-agent environments. 
Formal connections between EGT and MARL dates back to late '90s \cite{borgers1997learning}, where the authors demonstrated that, with a sufficiently small learning rate, the learning trajectories of \emph{cross learning}~\cite{cross1973stochastic} -- a stateless MARL algorithm -- converge to the trajectories of \emph{replicator dynamics}, a well-studied dynamical system in EGT.  This formal link has attracted increasing interest as it allows researchers to analyse MARL and its inherently stochastic learning processes through the deterministic framework of replicator dynamics. Building on this foundational work, researchers have extended this framework to more advanced MARL algorithms, such as Q-learning~\cite{tuyls2003selection, kianercy2012dynamics} and its variants~\cite{abdallah2008multiagent, klos2010evolutionary}, as well as regret minimization algorithms~\cite{kaisers2012common}. Other works, see, e.g., \cite{tuyls2003extended, kaisers2010frequency, perolat2021poincare}, have modified existing algorithms with known underlying dynamics, proposing new methods that inherit desirable properties from their associated evolutionary dynamics. Additionally, this framework enables the construction of learning algorithms for different types of games by analyzing the associated evolutionary dynamics~\cite{hennes2009state,vrancx2008switching,hennes2010resq}, which can also support function approximation and provide theoretical grounding for deep reinforcement learning~\cite{hennes2020neural,perolat2022mastering,barfuss2019deterministic}. 

Nevertheless, the majority of these studies have concentrated solely on replicator dynamics and its variants. Replicator dynamics, however, have been proven not to converge in certain game settings. To this end, we refer to the two main families of games as identified in~\cite{hofbauer2009stable}: \emph{strictly stable games} and \emph{null-stable games}. In strictly stable games, replicator dynamics can asymptotically converge to the Nash equilibrium, whereas in null stable games, the dynamics form closed orbits around the Nash equilibrium depending on the initial conditions. A particularly important class of null-stable games is zero-sum games. As a result, approaches based on replicator dynamics often rely on time-averaging to ensure convergence to the Nash equilibrium, which has been demonstrated to hold true~\cite{sandholm2008projection, hofbauer2009time, viossat2013no}. However, this method has a significant limitation because of the cumulative nature of the time average, which affects the ability of the dynamics to adapt to changes. Specifically, when the environment changes, the policy of the agents may require exponentially long periods to adapt to the new conditions. The problem becomes even more evident where the environment is subject to continuous changes, such is the case in feedback-evolving games~\cite{weitz2016oscillating, tilman2020evolutionary, stella2022lower, stella2023impact}. Indeed, the authors in \cite{zhang2023multi} also studied the learning trajectories of MARL in such a dynamic environment. In addition, in null-stable games, where the trajectories of continuous replicator dynamics converge to closed periodic orbits, discretizing the continuous dynamics can introduce diffusion. In turn, this diffusion can cause the trajectories to approach the boundary of the simplex of the policy, leading to numerical errors. This issue becomes significant when the Nash equilibrium is near the boundary of the simplex.

To overcome the limitations of replicator dynamics,  we turn our attention to learning algorithms based on alternative dynamics. Inspired by \cite{mertikopoulos2016learning,mertikopoulos2018riemannian}, where the mapping function is modified to switch the base dynamics to projection dynamics, we focus on innovative dynamics \cite{hofbauer2011deterministic}, a family of dynamics that includes BNN \cite{brown1950solutions} and Smith \cite{smith1984stability} dynamics. Innovative dynamics, in contrast to replicator dynamics, converge to the Nash equilibrium in null-stable games \cite{hofbauer2009stable}. However, their application to learning tasks is not as straightforward as with replicator dynamics. In dynamic environments with discrete changes and stochastic processes, replicator dynamics allow multi-step sampling to remain unbiased with respect to the underlying fitness-based equation. Hence, the need for the proposed alternative dynamics, as they are immune to this property, which prevents the use of similar update mechanisms.

\textit{Contribution}. In this paper, we introduce a novel algorithm, \emph{Experience-replay Innovative Dynamics (ERID)}, based on innovative dynamics through experience-replay. We show that the learning trajectories of ERID converge to three sets of innovative dynamics if we choose the corresponding protocol factor. This enables ERID to benefit from the convergence guarantees of these dynamics in strictly stable and null-stable games. Experience-replay is a reinforcement learning mechanism used to enhance learning efficiency and stability by storing and reusing past experiences during the training process~\cite{lange2012batch}. Its basic principle involves maintaining a memory buffer that records state transitions, actions, and rewards, which are then sampled randomly during training to decorrelate consecutive experiences and smooth the learning process \cite{mnih2015human}. In our case, experience-replay is used to reduce sample variance by mixing rewards from each step with historical rewards. This mitigates the influence of the non-linear revision protocol and ensures that the algorithm aligns with the desired underlying dynamics.

The remainder of this paper is structured as follows. Section~\ref{sec:prel} introduces some preliminary definition and notation used throughout the paper. In Section~\ref{sec:main}, we present the algorithm and the main theoretical results. In Section~\ref{sec:sim}, we show a set of simulations to evaluate the algorithm and corroborate the theoretical results. Finally, we provide concluding remarks and future directions of research in Section~\ref{sec:conc}.

\begin{wraptable}{L}{3cm}
\caption{RPS payoffs.}
\label{t:payoffs}
\begin{tabular}{c|ccc}
 & R & P & S \\\toprule
R & 0 & -$v$ & 1 \\  
P & $v$ & 0 & -1 \\  
S & -1 & 1 & 0 \\  \bottomrule
\end{tabular}
\end{wraptable}

\subsection{Motivating Example} \label{ME}
A simple example that highlights the weaknesses of algorithms based on time-average replicator dynamics is the nonstationary Rock-Paper-Scissors (RPS) game. This type of game was previously employed by Hennis \emph{et al.} to demonstrate the no-regret property of the NeuRD algorithm~\cite{hennes2020neural}. Building upon this framework, we modify the standard RPS game to illustrate our point. The basic rules of the modified RPS game are identical to the traditional RPS, where the winner of each round receives a payoff of +1, while the loser incurs a payoff of -1. If both players choose the same action, the outcome is a tie and, therefore, both players receiving a payoff of 0.

To introduce nonstationarity, a scaling factor $v > 1$ is applied to one of the matchups, amplifying the rewards of that particular confrontation. The corresponding payoffs are shown in Table~\ref{t:payoffs}. 


\begin{figure}[t]
    \centering
    \includegraphics[width=\linewidth]{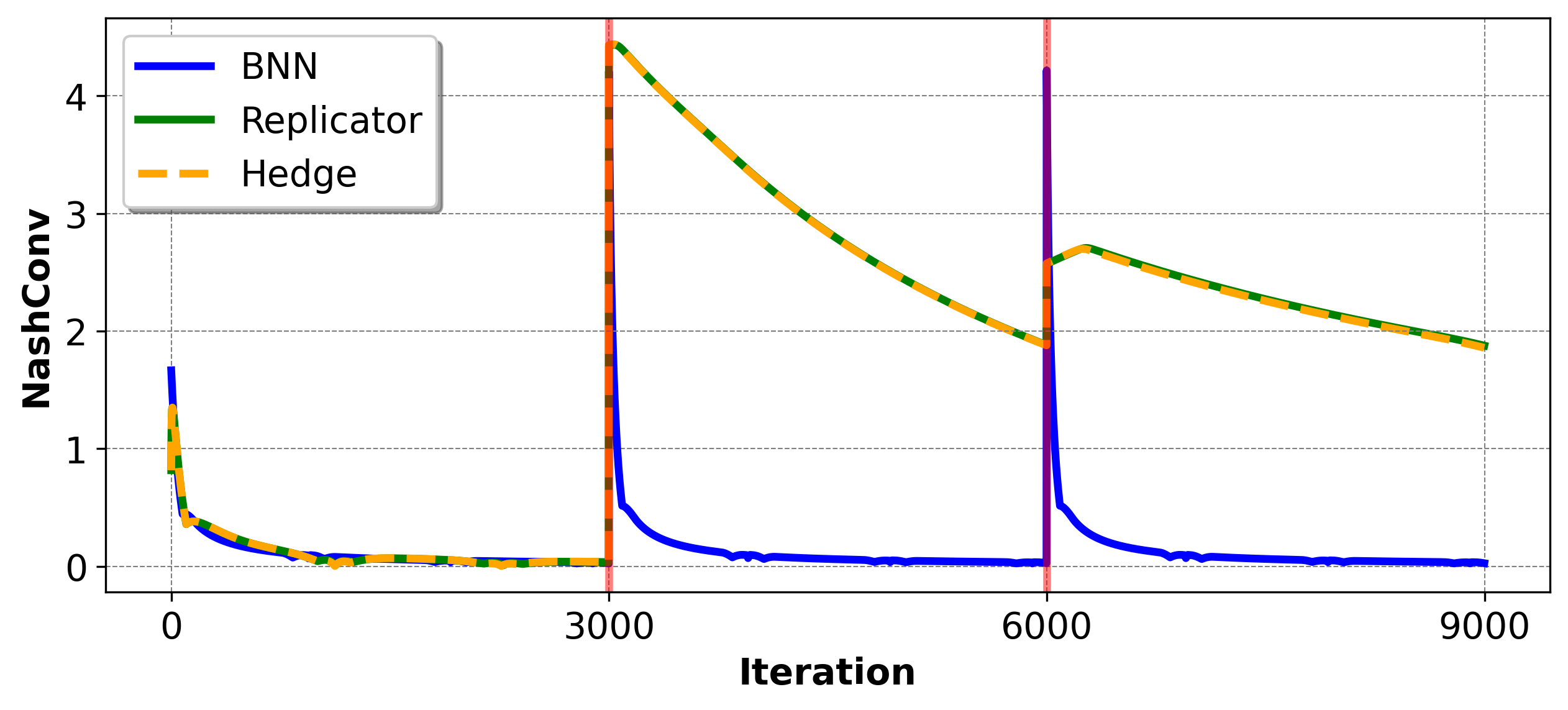}
    \caption{Policy {\sc NashConv} of BNN, Replicator and Hedge algorithms in nonstationary RPS, with the game phases every 3000 iterations separated by vertical red lines.}
    \label{fig:mot}
\end{figure}

To simplify the calculations, only one pair of matchups will be scaled in the game. The Nash equilibrium for this game is the following strategy distribution: the probability of selecting the unscaled action is \( \frac{v}{v+2} \), while the probability of selecting either of the scaled actions is \( \frac{1}{v+2} \) for each. In this example, we demonstrate the inability of replicator dynamics to adapt to dramatic changes. We set $v = 6$ and introduce a nonstationary environment. Initially, the rock-paper matchup is scaled. After 3000 time steps, the scaled matchup changes to scissors-rock, while rock-paper returns to being unscaled. After 6000 time steps, the scaled matchup shifts to paper-scissors, and scissors-rock returns to the unscaled state.

\ls{Figure~\ref{fig:mot} shows how {\sc NashConv}, a metric that measures the distance from a Nash equilibrium, changes over time for the BNN, the replicator, and the Hedge algorithms. The learning dynamics of the Hedge algorithm have been shown to follow replicator dynamics.} To highlight the shortcomings of methods based on time-averaging, two representative versions of replicator dynamics are included for comparison. \ls{At the start, we observe a fast drop in the {\sc NashConv} for all three sets of dynamics.} However, when the environment changes, BNN dynamics quickly adjust and continue reducing their Nash-Conv towards the new Nash equilibrium, while the two \ls{other two types of replicator dynamics slow down significantly. This becomes} even more evident \ls{after the second change in the environment}. The reason behind this issue \ls{can be easily explained by considering the nature of these algorithms: in the time-average dynamics, although the environment has changed}, all prior accumulated values remain centered around the previous Nash equilibrium. \ls{This results in a bias,} and simple calculations show that this bias is reduced to $1/(n+1)$ of its original size only after $n$ times more steps have passed since the change.





\section{Preliminaries}\label{sec:prel}
\ls{In this section, we first introduce preliminary definitions and notation before introducing the algorithm and theoretical results.}

\subsection{Game Theory}
\ls{In this paper, we consider two types of games: population games and a \(K\)-player normal form games. The latter is indicated by a superscript player index in square brackets, i.e., $x^{[i]}$ for a generic variable $x$.} A \emph{normal form game} (NFG) specifies the interaction of \(K\) players with corresponding action sets \(\{A^{[1]}, \ldots, A^{[K]}\}\). Each action set contains \(M \in \mathbb{N}\) actions \(A^{[k]} = \{a^{[k]}_1, \ldots, a^{[k]}_M\}\). The payoff function \(u : \prod_{k=1}^K A^{[k]} \to \mathbb{R}^K\) assigns a numerical utility to each player for each possible joint action \(\mathbf{a} = (a^{[1]}, \ldots, a^{[K]})\), where \(a^{[k]} \in A^{[k]}\) for all \(k \in \{1, \ldots, K\}\). \ls{The payoff function is derived from the corresponding entries of a payoff matrix}, denoted by \(P^{[i]}\) for each player \(i\). For any player \(i\), the reward resulting from a joint action \(\mathbf{a}\) is equal to the corresponding entry in the payoff matrix \(P^{[i]}\) \(r = \) \(P^{[i]}_{a^{[1]}, \ldots, a^{[K]}}\), where each index corresponds to the action chosen by each player. Let \(\pi^{[k]}\) denote the \(k\)-th player's mixed strategy. The expected utility for player \(k\) given strategy profile \(\pi = (\pi^{[1]}, \ldots, \pi^{[K]})\) is then \(\bar u^{[k]}(\pi) = \mathbb{E}_{\pi}[u^{[k]}(\mathbf{a}) | \mathbf{a} \sim \pi]\). The expected utility for the \(i\)-th action of player \(k\) is calculated as follows: \(u^{[k]}_i(\pi) = \mathbb{E}_{\pi}[u^{[k]}(a^{-[k]}, a^{[k]}_i)| a^{-[k]} \sim \pi^{-[k]}]\). In the notation \(a^{-[k]}\), the superscript \(-[k]\) indicates actions taken by all players except for player \(k\). \ls{A strategy profile $\pi_*^{[k]}$ is a Nash equilibrium if $\pi_*^{[k]} = BR^{[k]}(\pi_*^{-[k]})$, for all $k \in \{1,\dots,K\}$, where $BR^{[k]}(\cdot)$ denotes the best response for the $k$-th player, i.e., $BR^{[k]}(\pi^{-[k]}) := {\rm argmax}_{\pi^{-[k]}}[\bar u^{[k]}_i(\pi^{[k]},\pi^{-[k]})]$.}
\ls{To evaluate the quality of learned policies, we use the {\sc NashConv} metric~\cite{lanctot2017unified} defined as:
$$
\textsc{NashConv}(\pi) = \sum_k \bar u (BR^{[k]}(\pi^{-[k]}), \pi^{-[k]}) - \bar u^{[k]} (\pi),
$$
which can be interpreted as the distance of the learned policy $\pi$ to the Nash equilibrium. Thus, we are interested in lower values of {\sc NashConv}$(\pi)$.}

\subsection{Dynamics and Learning}
Replicator dynamics (RD) are the most well-known dynamics in EGT, taking the form of the following ODE in the single population setting: 
\begin{equation}\nonumber
\dot{x_i} = x_i[(Px)_i-x^TPx],
\end{equation}
where $x_i$ represents the frequency of strategy $i$ in the population and $P$ is the payoff matrix of the game under consideration. These dynamics describe how a population evolves over time through the fitness of certain strategies in the population. Replicator dynamics have also been used to describe the learning dynamics of multiple algorithms as discussed in the Introduction. The design of our algorithm takes inspiration from the use of RD in cross learning, a stateless reinforcement learning algorithm belonging to the category of finite action-set learning automata. This algorithm employs a policy iteration approach, starting with a random policy to facilitate exploration of the environment and learning from the actions taken. This policy is then updated in response to a reinforcement signal from the environment, allowing the agent to refine its strategy to maximise the expected reward.

At the onset of an epoch \( t \), the agent selects an action \( a(t) \) at random from the set of available actions \( \mathcal{A} \), guided by the current policy \( \pi(t) \). Following the selection of action \( a(t) \), the environment reinforces or deters the use of the same action through a reward signal \( r(t) \). The agent uses this reward \( r(t) \) to update its policy from \( \pi(t) \) to \( \pi(t+1) \), for each action \(i\), according to the following update rule:

\begin{eqnarray}
\pi_{i}(t+1) \leftarrow \pi_{i}(t) +
\begin{cases}\label{eq:policy}
    \alpha r(t)(1-\pi_{i}(t)),  & \text{if $a(t) = i$}, \\
    -\alpha r(t)\pi_{i}(t), & \text{otherwise}.
\end{cases}
\end{eqnarray} 

For completeness, we provide the following lemma, which shows the convergence of cross learning to the two-player RD. 

\begin{lemma}
Let us consider a cross learning model in a 2-player normal form game. If $\alpha \rightarrow 0$, then the trajectory of cross learning converges to the trajectory of the two-player normal form game resulting from replicator dynamics.
\end{lemma}
\begin{proof}
For brevity, we refer the reader to \cite{borgers1997learning}.
\end{proof}

\section{Experience-replay Innovative Dynamics}\label{sec:main}
In this section, we introduce \ls{the main contibution of this paper, namely,} experience-replay innovative dynamics (ERID), a stateless reinforcement learning framework \ls{based on three types of dynamics within the family of innovative dynamics.} 

\subsection{Experience Replay}
\ls{Before we present the framework, we introduce experience-replay as it constitutes the main component of our algorithm.} To implement experience-replay, suppose a buffer of size \(K \in \mathbb{N}\) is used to store \ls{samples} from the last \(K\) iterations. At time \(t\), the rewards from interactions are stored in \(\mathcal{B} = \{b_t, b_{t-1}, \ldots, b_{t-K+1}\}\). The action information is stored as sets of indices \(\mathcal{I} = \{I_i \,:\, I_i \subseteq \{t, t-1, \ldots, t-K+1\}\}_{i \in \{1, \ldots, M\}}\). Each set \(I_i\) corresponds to a particular action \(a_i\). If action \(a_i\) was taken at time step \(m\), then \(m\) is included in the set \(I_i\), and \(b_m\) is the reward associated with action \(a_i\). These sets are mutually exclusive, meaning that for any \(i \neq j\), \(I_i \cap I_j = \emptyset\). Additionally, \(\bigcup_{i=1}^{M} I_i = \{t, t-1, \ldots, t-K+1\}\) ensures a comprehensive yet distinct categorisation of action indices.

\ls{At} each time step, the agent \ls{chooses an action} and stores the \ls{corresponding action-reward tuple} in the buffer. Due to the fixed size of the buffer, \ls{when the buffer is full, the oldest data are discarded to make space for the new entries while the data in the buffer are} used to update the agent's policy. This process allows the agent to refine its policy based on past experiences, improving its performance.

\subsection{Experience Replayed Cross Learning}
In this section, we introduce our model for ERID. A critical part of this model involves calculating the average rewards associated with each action, which are used to refine the agent's policy.

\ls{In order to use the stored experiences for the policy updates, it is essential to calculate the average rewards associated with each action and the overall average reward across all actions. These two quantities are defined as:}
\begin{align}\label{eq:rewards}
    \bar{r}_i = \begin{cases}
    \displaystyle \frac{\sum\limits_{j \in I_i} b_j}{|I_i|}, & \quad {\rm if} \; I_i \neq \emptyset,\\
    0, & \quad {\rm otherwise},
    \end{cases} \quad \quad
    \displaystyle \bar{r} = \frac{\sum\limits_{j=1}^{K} b_j}{K},
\end{align}
\ls{respectively. In evolutionary} game theory, each revision protocol $\rho_{ij}$ corresponds to a specific \ls{set of evolutionary dynamics, and determines how the probabilities of choosing different strategies change} over time. In our reinforcement learning framework, we introduce the protocol factor $\eta_{ij}$, which is computed based on the average rewards \ls{as defined in~\eqref{eq:rewards}}. For each specific \ls{set of dynamics}, the corresponding $\rho_{ij}$ can be mapped to a specific $\eta_{ij}$. The specific mapping involves replacing the fitness values in \(\rho_{ij}\) with the corresponding reward values. We can now present the update formula for the policy $\pi_i(t)$:
\begin{equation}\label{eq:ERCL}
\pi_i(t+1) \leftarrow \pi_i(t) + \alpha \left( \sum_{j=1}^{M}\pi_j(t) \eta_{ji} - \pi_i(t) \sum_{j=1}^{M} \eta_{ij}\right),
\end{equation}
where $\alpha$ is the learning rate. \ls{The pseudocode of the algorithm corresponding to the above policy update is given} in Algorithm~\ref{alg:ERID}.
\color{black}

\subsection{ERID with BNN Dynamics}
BNN dynamics were first introduced by Brown and von Neumann~\cite{brown1950solutions}, who demonstrated their global convergence to equilibrium sets in certain special cases of zero-sum games. Later, in \cite{hofbauer2009stable}, two related properties were proven: zero-sum games belong to the class of null-stable games, and BNN dynamics globally converge to the equilibrium set in null-stable games.
The BNN dynamics is defined as follows:
\[
\dot{x}_i = \hat{a}_i(x) - x_i \sum_{j=1}^{n} \hat{a}_j(x),
\]
where
\[
\hat{a}_i(x) = [ (Px)_i - x^T P x ]_+,
\]
and \([\cdot]_+\) denotes the positive part.
Under the two-player normal form game setting, the dynamics turn to be:
\begin{align}
    \dot{x}_i &= [ (P^{[1]} y)_i - x^T P^{[1]} y ]_+ - x_i \sum_{j=1}^{n} [ (P^{[1]} y)_j - x^T P^{[1]} y ]_+ ,\label{eq:BNN1}\\
    \dot{y}_i &= [ (P^{[2]} x)_i - y^T P^{[2]} x ]_+ - y_i \sum_{j=1}^{n} [ (P^{[2]} x)_j - y^T P^{[2]} x ]_+ .\label{eq:BNN2}
\end{align}
To align the ERID algorithm with the BNN dynamics, we set \(\eta_{ij}\) as follows:
\begin{equation}
\eta_{ij} = \left[ \bar{r}_j - \bar{r} \right]_+.\nonumber
\end{equation}

By incorporating \(\eta_{ji}\) into the policy update equation (\ref{eq:ERCL}), we derive the following update rule:
\begin{equation}\label{eq:bnnupdate}
\pi_i(t+1) \leftarrow \pi_i(t) + \alpha \left([\bar{r}_i - \bar{r}]_+ - \pi_i(t) \sum_{j=1}^{M} [\bar{r}_j - \bar{r}]_+\right).
\end{equation}
\ls{By using equation~\eqref{eq:bnnupdate} in line 9 of Algorithm~\ref{alg:ERID}, we derive the ERID corresponding to} the BNN protocol factor. \ls{The next result shows the convergence of the ERID algorithm to the BNN dynamics when line 9 uses the policy update as defined in equation~\eqref{eq:bnnupdate}.}

\renewcommand{\algorithmicrequire}{\textbf{Input:}}
\renewcommand{\algorithmicensure}{\textbf{Output:}}
\begin{algorithm}[t]
\caption{Experience-replay Innovative Dynamics}
\label{alg:ERID}
\begin{algorithmic}[1]
\Require Initial policy $\pi_0$, buffer size $K$, learning rate $\theta$
\State \textbf{Initialize:} Policy $\pi_0 \gets$ initial strategy, buffer $B \gets$ empty
\For{$t = 1, 2, \dots$}
    \If{$t > K$}
        \State $B \gets \text{shift}(B, (a, r))$
    \Else
        \State $B \gets B \cup (a, r)$
    \EndIf
    \State $\bar{r}_i \gets \text{getAverageReward}(B, i)$
    \State $\bar{r} \gets \text{getOverallAverageReward}(B)$
    \State $\pi_t \gets \text{updatePolicy}(\pi_t, \bar{r}_i, \bar{r}, \theta)$
\EndFor
\end{algorithmic}
\end{algorithm}

\begin{theorem}\label{th:BD}
Consider a two-player normal form game, where both use ERID with BNN protocol factor update rule (\ref{eq:bnnupdate}). While $\alpha K \rightarrow 0$ and $K\rightarrow \infty$, the trajectory of the policy retrieved by ERID converges to the trajectory of the game dynamics resulting from the BNN dynamics.
\end{theorem}

\begin{proof}
    Let \(R^{[i]}(a^{[1]}, a^{[2]})\) denote the reward function of the \(i\)-th agent. If \(\pi_i^{[1]} > 0\), the expected average reward of action \(i\) for agent 1 can be derived as:
\begin{align}
E(\bar{r}^{[1]}_i) &= \frac{\sum_{j \in I^{[1]}_i} E(b^{[1]}_j)}{E(|I^{[1]}_i|)} \nonumber \\
&= \frac{\sum_{j=0}^{K-1} \pi^{[1]}_i(t-j)\left(\sum_{m=1}^{M^{[2]}}\pi^{[2]}_m(t-j)R^{[1]}(a^{[1]}_i,a^{[2]}_m)\right)}{\sum_{j=0}^{K-1} \pi^{[1]}_i(t-j)}. \label{eq:expected_value}
\end{align}

We use \(\delta\) to denote the change in an agent's policy:
\begin{align}
\alpha\delta^{[i]}_m(t) &= \pi^{[i]}_m(t+1) - \pi^{[i]}_m(t), \nonumber \\
\delta^{[i]}_m(t) &= [\bar{r}^{[i]}_m - \bar{r}^{[i]}]_+ - \pi^{[i]}_m(t) \sum_{j=1}^{M} [\bar{r}^{[i]}_j - \bar{r}^{[i]}]_+. \nonumber
\end{align}

Given the settings of a normal form game, the reward function \(R\) is bounded, hence there exist a maximum \(R_{\text{max}}^{[i]}\) and a minimum \(R_{\text{min}}^{[i]}\) values, where
\(R_{\text{max}}^{[i]} = \max_{m \in M^{[1]}, n \in M^{[2]}} R^{[i]}(a_m^{[1]}, a_n^{[2]})\) and
\(R_{\text{min}}^{[i]} = \min_{m \in M^{[1]}, n \in M^{[2]}} R^{[i]}(a_m^{[1]}, a_n^{[2]})\). Additionally, the policy \(\pi\) values lie between 0 and 1. Therefore, we have:
\begin{equation}
|\delta^{[i]}| < M^{[i]}(R_{\text{max}}^{[i]} - R_{\text{min}}^{[i]}) < C ,\nonumber
\end{equation}
where \(C\) is a constant representing this bound.

Then, using \(\delta\) and \(\pi^{[i]}_m(t)\), we can represent \(\pi^{[i]}_m(t-j)\) as:
\begin{align*}
    \pi^{[i]}_m(t-j) &= \pi^{[i]}_m(t) - \alpha\sum_{l=1}^{j}\delta_m^{[i]}(t-l), \\
    \pi^{[i]}_m(t) - &\pi^{[i]}_m(t-j) = \alpha\sum_{l=1}^{j}\delta_m^{[i]}(t-l), \\
    |\pi^{[i]}_m(t) - &\pi^{[i]}_m(t-j)| \leq j\alpha C.
\end{align*} 

If we consider a sufficiently small learning rate \(\alpha\), it follows that:
\begin{align*}
    \lim_{j\alpha\rightarrow0}|\pi^{[i]}_m(t) - \pi^{[i]}_m(t-j)| &= 0.
\end{align*}

Then the equation (\ref{eq:expected_value}) can be derived as:
\begin{align}
E(\bar{r}^{[1]}_i) &=\frac{\sum_{j=0}^{K-1} \pi^{[1]}_i(t-j)\left(\sum_{m=1}^{M^{[2]}}\pi^{[2]}_m(t-j)R^{[1]}(a^{[1]}_i,a^{[2]}_m)\right)}{\sum_{j=0}^{K-1} \pi^{[1]}_i(t-j)}, \nonumber \\
\lim_{K\alpha\rightarrow0}E(\bar{r}^{[1]}_i) &= \frac{K \pi^{[1]}_i(t) \sum_{m=1}^{M^{[2]}}\pi^{[2]}_m(t)R^{[1]}(a^{[1]}_i,a^{[2]}_m)}{K\pi^{[1]}_i(t)}, \nonumber \\
\lim_{K\alpha\rightarrow0}E(\bar{r}^{[1]}_i) &= \sum_{m=1}^{M^{[2]}}\pi^{[2]}_m(t)R^{[1]}(a^{[1]}_i,a^{[2]}_m). \nonumber
\end{align}

Analogous proofs can be used for \(\bar{r}^{[1]}, \bar{r}_i^{[2]}, \bar{r}^{[2]}\). We have:
\begin{align}
\lim_{K\alpha\rightarrow0}E(\bar{r}^{[2]}_i) &= \sum_{m=1}^{M^{[1]}}\pi^{[1]}_m(t)R^{[2]}(a^{[1]}_m,a^{[2]}_i), \nonumber \\
\lim_{K\alpha\rightarrow0}E(\bar{r}^{[1]}) &= \sum_{m=1}^{M^{[1]}}\sum_{n=1}^{M^{[2]}}\pi^{[1]}_m(t)\pi^{[2]}_n(t)R^{[1]}(a^{[1]}_m,a^{[2]}_n), \nonumber \\
\lim_{K\alpha\rightarrow0}E(\bar{r}^{[2]}) &= \sum_{m=1}^{M^{[1]}}\sum_{n=1}^{M^{[2]}}\pi^{[1]}_m(t)\pi^{[2]}_n(t)R^{[2]}(a^{[1]}_m,a^{[2]}_n). \nonumber
\end{align}

\ls{By} the law of large numbers, as \(K \rightarrow \infty\), \(\bar{r}_i^{[1]}\) converges to its expected value, \(E(\bar{r}_i^{[1]})\). Similarly, \(\bar{r}_i^{[2]}\), \(\bar{r}^{[1]}\), and \(\bar{r}^{[2]}\) also converge to their respective expected values, \(E(\bar{r}_i^{[2]})\), \(E(\bar{r}^{[1]})\), and \(E(\bar{r}^{[2]})\). Under the conditions \(K\alpha\rightarrow 0\) and \(K\rightarrow \infty\), the difference between \(\pi^{[1]}_i(t+1)\) and \(\pi^{[1]}_i(t)\) can be derived as:

\begin{align}
\delta^{[1]}_i(t) &= \left[ \bar{r}^{[1]}_i - \bar{r}^{[1]} \right]_+ 
- \pi^{[1]}_i(t) \sum_{j=1}^{M^{[1]}} \left[ \bar{r}^{[1]}_j - \bar{r}^{[1]} \right]_+ \nonumber\\
&= \Bigg[ \sum_{m=1}^{M^{[2]}} \pi^{[2]}_m(t) R^{[1]}(a^{[1]}_i, a^{[2]}_m) \nonumber\\
&\quad - \sum_{m=1}^{M^{[1]}} \sum_{n=1}^{M^{[2]}} \pi^{[1]}_m(t) \pi^{[2]}_n(t) R^{[1]}(a^{[1]}_m, a^{[2]}_n) \Bigg]_+ \nonumber\\
&\quad - \pi^{[1]}_i(t) \sum_{j=1}^{M^{[1]}} \Bigg[ \sum_{m=1}^{M^{[2]}} \pi^{[2]}_m(t) R^{[1]}(a^{[1]}_j, a^{[2]}_m) \nonumber\\
&\quad\quad - \sum_{m=1}^{M^{[1]}} \sum_{n=1}^{M^{[2]}} \pi^{[1]}_m(t) \pi^{[2]}_n(t) R^{[1]}(a^{[1]}_m, a^{[2]}_n) \Bigg]_+.
\end{align}

In the context of a normal form game, the reward function \(R^{[i]}(a^{[1]}_m, a^{[2]}_n)\) can be considered as the entry of the corresponding payoff matrix:
\[
R^{[i]}(a^{[1]}_m, a^{[2]}_n) = p^{[i]}_{mn}.
\]
Thus, the update rule for \(\delta^{[1]}_i(t)\) simplifies to:
\begin{align}
\delta^{[1]}_i(t) &= \left[(P^{[1]}\pi^{[2]}(t))_i - \pi^{[1]}(t)^T P^{[1]}\pi^{[2]}(t)\right]_+\nonumber\\ -\pi^{[1]}&_i(t)\sum_{j=1}^{M^{[1]}}\left[(P^{[1]}\pi^{[2]}(t))_j - \pi^{[1]}(t)^T P^{[1]}\pi^{[2]}(t)\right]_+.\nonumber
\end{align}

To derive a continuous-time limit of the learning model, we assume that the time interval between two repetitions of the game is represented by a parameter $\theta$, where $0 < \theta \leq 1$. After each repetition of the game, players adjust their states by a factor proportional to $\theta$. The key assumption is that the rate at which players adjust their states decreases in proportion to the shrinking time interval between repetitions, \ls{in a similar manner to the} method used in \cite{borgers1997learning}. \ls{Formally}, it is expressed as:
\[
\pi(t+\theta) = \pi(t) + \theta \delta(t)\nonumber,
\]
where $\pi(t)$ represents the state of the player at time $t$, and $\delta(t)$ is the adjustment made to the player's state at time $t$.

Under this assumption, we can consider the discrete update learning algorithm as a continuous function and derive \ls{the corresponding differential equation as}:
\[
\lim_{\theta \to 0} \frac{\pi(t+\theta) - \pi(t)}{\theta} = \delta(t).
\]

Therefore:
\begin{align}
\frac{d\pi^{[1]}_i(t)}{dt} &= \delta^{[1]}_i(t) \nonumber\\
&= \left[(P^{[1]}\pi^{[2]}(t))_i - \pi^{[1]}(t)^T P^{[1]}\pi^{[2]}(t)\right]_+ - \nonumber\\
\pi^{[1]}&_i(t)\sum_{j=1}^{M^{[1]}}\left[(P^{[1]}\pi^{[2]}(t))_j - \pi^{[1]}(t)^T P^{[1]}\pi^{[2]}(t)\right]_+.\label{eq:BNNL1}
\end{align}
Under the same game settings, if we consider the policies \(\pi^{[1]}\) and \(\pi^{[2]}\) to be analogous to the strategy proportions \(x\) and \(y\), then equation (\ref{eq:BNNL1}) is equivalent to equation (\ref{eq:BNN1}).

Using a similar method, we can derive:
\begin{align}
\frac{d\pi^{[2]}_i(t)}{dt} &= \delta^{[2]}_i(t) \nonumber\\
&= \left[(P^{[2]}\pi^{[1]}(t))_i - \pi^{[2]}(t)^T P^{[2]}\pi^{[1]}(t)\right]_+ - \nonumber\\
\pi^{[2]}&_i(t)\sum_{j=1}^{M^{[2]}}\left[(P^{[2]}\pi^{[1]}(t))_j - \pi^{[2]}(t)^T P^{[2]}\pi^{[1]}(t)\right]_+.\label{eq:BNNL2}
\end{align}

Similarly, under the same conditions, there is also an equivalence between equation (\ref{eq:BNNL2}) and Equation (\ref{eq:BNN2}).
 This shows that the dynamics of ERID under the conditions of \(\alpha K \rightarrow 0\) and \(K \rightarrow \infty\) are equivalent to the BNN dynamics. This concludes the proof.
 \end{proof}

\subsection{ERID with Smith Dynamics}
Smith dynamics, introduced by Smith \cite{smith1984stability} in the transportation literature, is a dynamic model specifically designed for congestion games. \ls{A proof of global stability for these dynamics to the set of equilibria in null-stable games is given in~\cite{smith1984stability}}. Since the revision protocol in Smith dynamics involves comparing every pair of strategies, it falls under the category of pairwise comparison dynamics. Due to this pairwise comparison characteristic, Smith dynamics tends to react more quickly to changes in dynamic targets compared to BNN dynamics; however, it also exhibits greater fluctuations near \ls{the} equilibrium. Smith dynamics \ls{are} defined as follows:
\begin{align}\nonumber
\dot{x}_i = \sum_{j=1}^{n} x_j [(Px)_i - (Px)_j]_+ - x_i \sum_{j=1}^{n} [(Px)_j - (Px)_i]_+. 
\end{align}

\ls{To account for the above protocol in our algorithm}, we set:

\begin{equation}\nonumber
\eta_{ij} = \left[ \bar{r}_j - \bar{r}_i \right]_+.
\end{equation}

By \ls{substituting the above protocol factor in place of} $\eta_{ji}$ into the policy update equation~(\ref{eq:ERCL}), we derive the following update rule:
\begin{equation}
\pi_i(t+1) \leftarrow \pi_i(t) + \alpha \left( \sum_{j=1}^{M}\pi_j(t) \left[ \bar{r}_i - \bar{r}_j \right]_+ - \pi_i(t) \sum_{j=1}^{M} \left[ \bar{r}_j - \bar{r}_i \right]_+\right).\label{eq:ERCLPC}
\end{equation}
\ls{In a similar manner to the BNN dynamics, by using the above protocol factor in line 9 of Algorithm~\ref{alg:ERID}, we derive the corresponding algorithm with the Smith protocol factor, and establish the following result}.
\begin{theorem}\label{th:PC}
Consider a normal form game with two players, both of them use ERID with Smith update rule~\eqref{eq:ERCLPC}. If $\alpha K \rightarrow 0$ and $K\rightarrow \infty$, then the trajectory of the policy retrieved by ERID converges to the trajectory of the game dynamics resulting from the Smith dynamics.
\end{theorem}
The proof of Theorem~\ref{th:PC} is very similar to that of Theorem~\ref{th:BD}. It only requires substituting Equation~\ref{eq:bnnupdate} with Equation~\ref{eq:ERCLPC}. For the sake of brevity, the proof is omitted here.

\subsection{ERID with Smith-replicator-based Pairwise Dynamics}
\ls{Finally, we show that our algorithm} is also compatible with dynamics defined by \ls{other} revision protocols, such as the Smith-replicator-based pairwise dynamics \cite{barreiro2015decentralized}. 
\ls{This model} was developed for environments where constraints are imposed on the proportion of strategies. Assuming $\underline{x}_i$ and $\overline{x}_i$ represent the lower and upper bounds on $x_i$, respectively, the Smith-replicator-based pairwise dynamics can be expressed as follows:

\begin{align*}
\dot{x}_i = \sum_{j=1}^{n} (x_j - \underline{x}_j)(\overline{x}_i - x_i)[(Px)_i - (Px)_j]_+ \\ - \sum_{j=1}^{n} (x_i - \underline{x}_i)(\overline{x}_j - x_j)[(Px)_j - (Px)_i]_+.
\end{align*}

In the context of MARL, these constraints are applied to the frequency of actions selected \ls{by the agent as part of their policy. To integrate our algorithm with the Smith-replicator-based pairwise dynamics, we define the protocol factor $\eta_{ij}$ as}:
\begin{equation}
    \eta_{ij} = \frac{1}{\pi_i}(\pi_j - \underline{\pi}_j)(\overline{\pi}_i - \pi_i)\left[ \bar{r}_j - \bar{r}_i \right]_+.
\end{equation}

By incorporating \(\eta_{ji}\) into the policy update equation (\ref{eq:ERCL}), we derive the following update rule:
\begin{align}\label{eq:ERCLSP}
\pi_i(t+1) &\leftarrow \pi_i(t) + \alpha \Big( \sum_{j=1}^{M}(\pi_j(t) - \underline{\pi}_j)(\overline{\pi}_i - \pi_i(t)) \left[ \bar{r}_i - \bar{r}_j \right]_+ \nonumber\\&- (\pi_i(t) - \underline{\pi}_i)(\overline{\pi}_j - \pi_j(t)) \left[ \bar{r}_j - \bar{r}_i \right]_+ \Big).
\end{align}
\ls{In the following result, we prove that the trajectory of the policy of our algorithm converges to the Smith-replicator-based pairwise dynamics when equation~\eqref{eq:ERCLSP} is substituted in line 9 of Algorithm~\ref{alg:ERID}}.
\begin{theorem}\label{th:SP}
Consider a normal form game with two players, both of them use ERID with Smith-replicator-based pairwise revision protocol update rule (\ref{eq:ERCLSP}). If $\alpha K \rightarrow 0$ and $K\rightarrow \infty$, then the trajectory of the policy retrieved by ERID converges to the trajectory of the game dynamics resulting from the Smith-replicator-based pairwise dynamics.
\end{theorem}
For the same reasons as in Theorem~\ref{th:PC}, the proof of Theorem~\ref{th:SP} is omitted.

\section{Evaluation}\label{sec:sim}
In this section, we design two sets of experiments. The first set aims to validate the theoretical results presented in Section~\ref{sec:main}, using the two-player matching pennies game and the biased RPS game. The second set is designed to compare ERID with algorithms based on replicator dynamics, utilizing the non-stationary RPS game.


\subsection{Comparison Dynamics with ERID}
For the evolutionary dynamic trajectories, we use a fixed payoff matrix, where the payoff of strategies within one population is influenced by the real-time proportions of strategies employed by the \ls{other} population. In the analysis of learning dynamics, two agents equipped with identical learning algorithms \ls{play} against each other, with each agent's policy recorded at every timestep. These trajectories are then depicted and compared to illustrate the effectiveness of ERID in capturing the dynamics of the game.

\textbf{Matching Pennies}. Matching Pennies is a classic two-player zero-sum game. Each player chooses to display either heads or tails simultaneously without knowledge of the \ls{other player's} choice. The objective is simple: if both players reveal the same side, Player 1 wins; if the sides differ, Player 2 wins. The \ls{payoff matrix of each player is as follows:}

\begin{equation}
A = \begin{pmatrix}
      1 & -1 \\
      -1 & 1
  \end{pmatrix}
\quad
B = \begin{pmatrix}
      -1 & 1 \\
      1 & -1
  \end{pmatrix}
\label{eq:MPpay}
\end{equation}

The Nash Equilibrium in the Matching Pennies game occurs when both players randomise their choices equally between heads and tails, each with a probability of 50\%.  In Figure~\ref{fig:combined}, \ls{we set} five distinct initial \ls{conditions, namely,} (0.5, 0.6), (0.2, 0.2), (0.3, 0.7), (0.8, 0.2), and (0.9, 0.9). The color gradients along the trajectories represent the progression of time, with matching colours on both trajectories indicating the strategies adopted by the players at the same \ls{timestep}.

\begin{figure}[t]
    \centering
    \begin{subfigure}[b]{0.5\linewidth}
        \includegraphics[width=\linewidth]{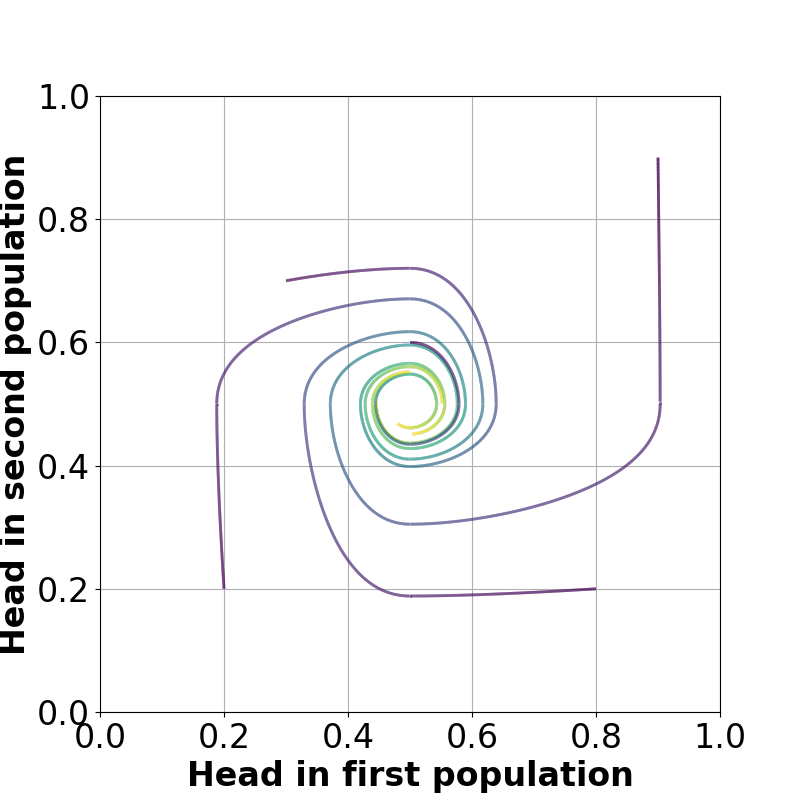}
        \label{fig:fig1a}
    \end{subfigure}
    \hspace{-0.02\linewidth} 
    \begin{subfigure}[b]{0.5\linewidth}
        \includegraphics[width=\linewidth]{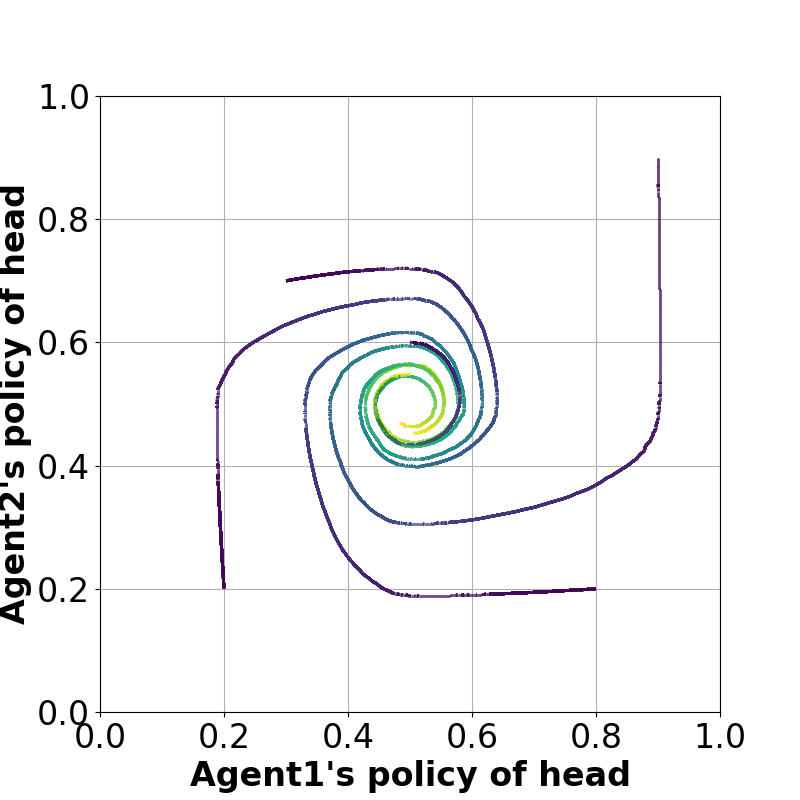}
        \label{fig:fig1b}
    \end{subfigure}
    
    \vspace{-0.08\linewidth}

    \begin{subfigure}[b]{0.5\linewidth}
        \includegraphics[width=\linewidth]{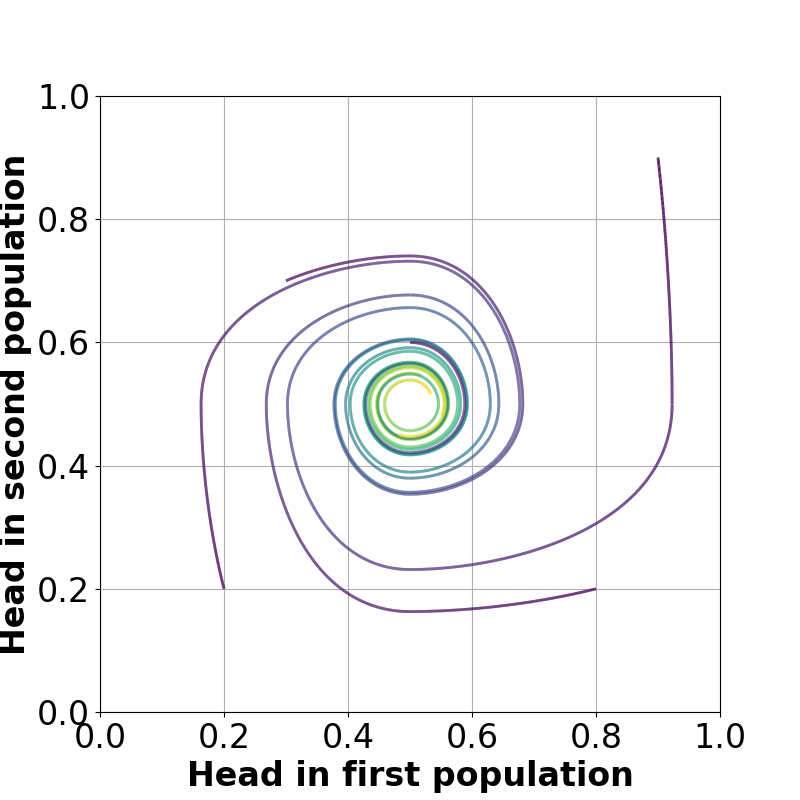}
    \end{subfigure}
    \hspace{-0.02\linewidth} 
    \begin{subfigure}[b]{0.5\linewidth}
        \includegraphics[width=\linewidth]{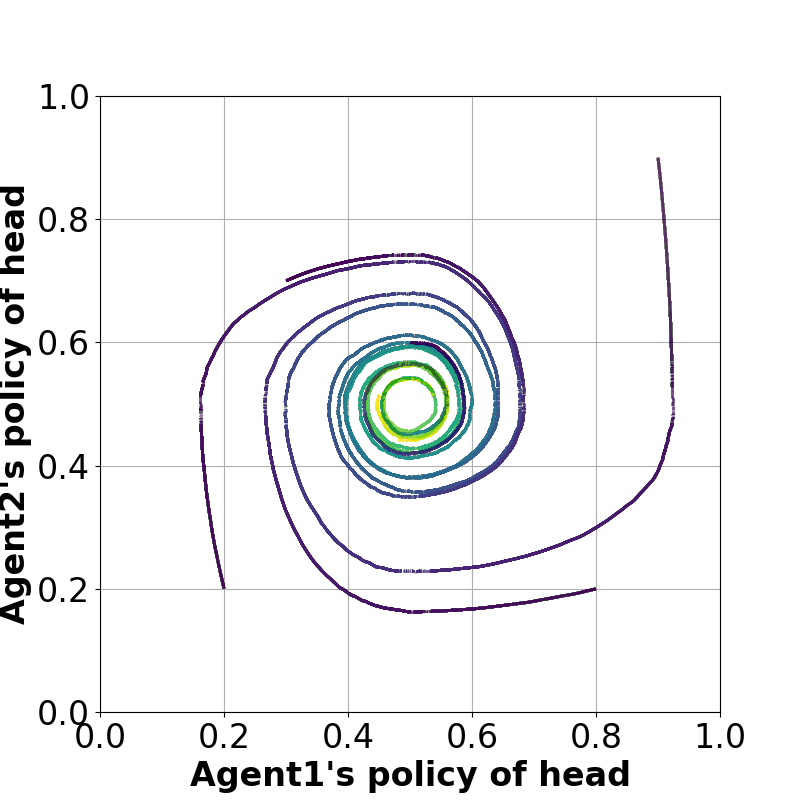}
    \end{subfigure}
    
    \caption{Innovative dynamics (left) vs policy trajectories (right) of the ERID algorithm using BNN dynamics (top) and Smith dynamics (bottom), respectively.}
    \label{fig:combined}
\end{figure}

\textbf{Biased RPS}. \ls{Here, we show the dynamics for the Biased RPS game}, a simplified version of the nonstationary RPS used in \ls{the Introduction}. In the nonstationary version, both the scaled matchups and the scaling factor could change, whereas in this version, both are fixed: the rock-paper matchup is consistently scaled by a factor of 2. The payoff matrix \ls{of each player is shown in the following:}

\begin{equation}
A = B = \begin{pmatrix}
    0 & -1 & 1 \\
    1 & 0 & -2 \\
    -1 & 2 & 0
\end{pmatrix}
\label{eq:rpspay}
\end{equation}

Figure~\ref{fig:BNN_RPS_combined} \ls{depicts} the results of \ls{the} simulations within a biased RPS game using two different dynamics: BNN dynamics for the top plots and Smith dynamics for the bottom plots. \ls{For this game, we use barycentric coordinates, as depicted in form of the shown triangles}, representing the strategy combinations of each player throughout the simulation. Each plot includes two trajectories, corresponding to the strategies chosen by the two players. The color gradients are used in the same way as in the matching pennies game. \ls{In this example, we set the initial conditions to (0.1, 0.1, 0.8) and (0.8, 0.1, 0.1), respectively}.

Both Figure~\ref{fig:combined} and Figure~\ref{fig:BNN_RPS_combined} share the same structure: the top plots depict BNN dynamics, while the bottom plots show Smith dynamics. In each figure, the left-hand plots present trajectories derived from the \ls{innovative} dynamics, whereas the right-hand plots display the results from simulations using the ERID algorithm with a step size of 1e-5 and a buffer size of 1000.

In each pair of plots, we observe that the ERID-generated trajectories on the right closely follow the \ls{evolutionary} trajectories on the left. Since both BNN and Smith dynamics are known to converge to the Nash equilibrium in zero-sum games, we can see that the corresponding ERID trajectories also converge to the NE.

Although the trajectories on the right-hand side exhibit slight perturbations due to the stochastic nature of reinforcement learning, leading to less smooth curves compared to the left-hand side, the overall trends remain remarkably consistent. \ls{This is in line with the results in} Theorem~\ref{th:BD} and Theorem~\ref{th:PC}, and show that ERID can benefit from the convergence guarantees provided by the respective dynamics.

\subsection{Comparison Cross Learning with ERID}
In this chapter, we revisit the example from the beginning of the paper and further increase its complexity. We continue using the nonstationary RPS game, but now the game settings change in a continuous manner. At the start of the game, the rock-paper matchup is scaled by a factor of 6. From step 3e5 to 9e5, this scaling factor gradually decreases to 1. Following this, from step 9e5 to 15e5, the scaling factor is applied to the scissors-rock matchup, gradually increasing from 1 to 6, and then over the next 6e5 steps, it decreases back to 1. Finally, for the last 6e5 steps, the scale factor is applied to the paper-scissors matchup, increasing gradually from 1 to 6.

We consider the case of a single agent playing a self-play game, a setup that is sufficient to ensure convergence to the Nash equilibrium in symmetric zero-sum games. The algorithms used include ERID with BNN, ERID with Smith, and Cross Learning. We selected Cross Learning as the representative because other algorithms either require global information (e.g., Hedge), introduce some bias relative to replicator dynamics (e.g., EXP3, FAQ-learning), or are based on function approximation (e.g., NeuRD). Given the underlying evolutionary dynamics are the same, we believe Cross Learning sufficiently represents the class of learning algorithms.

Due to environmental shifts that alter the range of {\sc NashConv} values, we also present both {\sc NashConv} and relative {\sc NashConv}. The relative {\sc NashConv} is calculated as the {\sc NashConv} divided by a scaling factor, which represents its potential maximum value.

\begin{figure}[b]
    \centering
    \begin{subfigure}[b]{0.5\linewidth}
        \includegraphics[width=\linewidth]{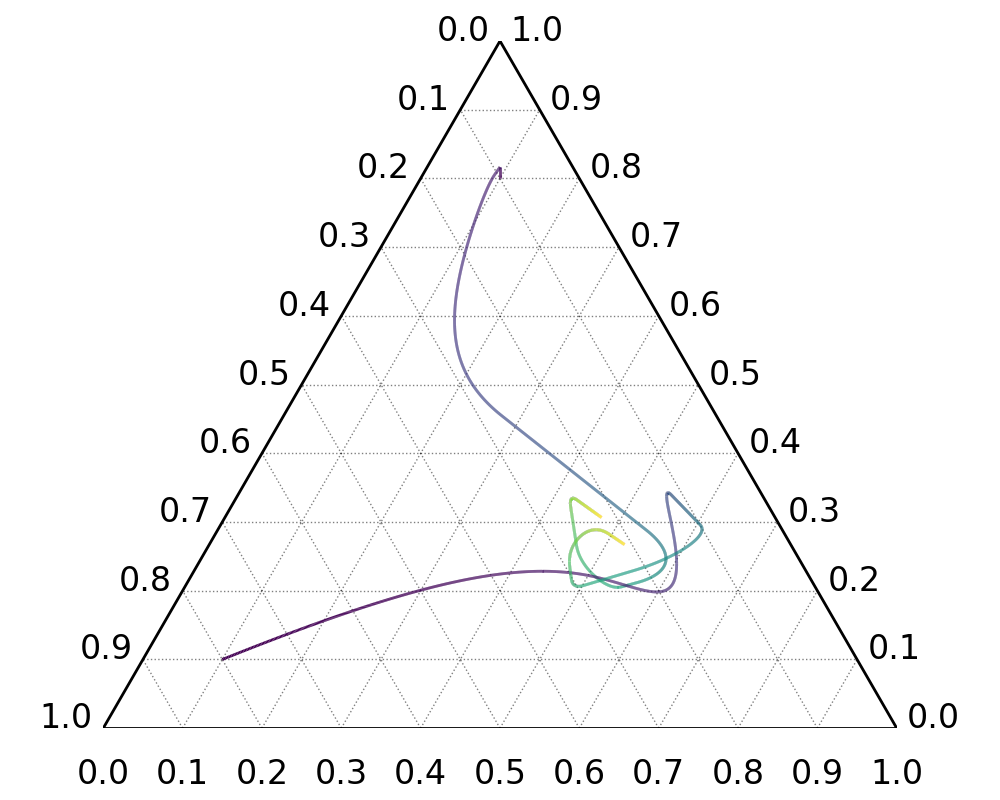}
        \label{fig:fig2a}
    \end{subfigure}
    \hspace{-0.02\linewidth} 
    \begin{subfigure}[b]{0.5\linewidth}
        \includegraphics[width=\linewidth]{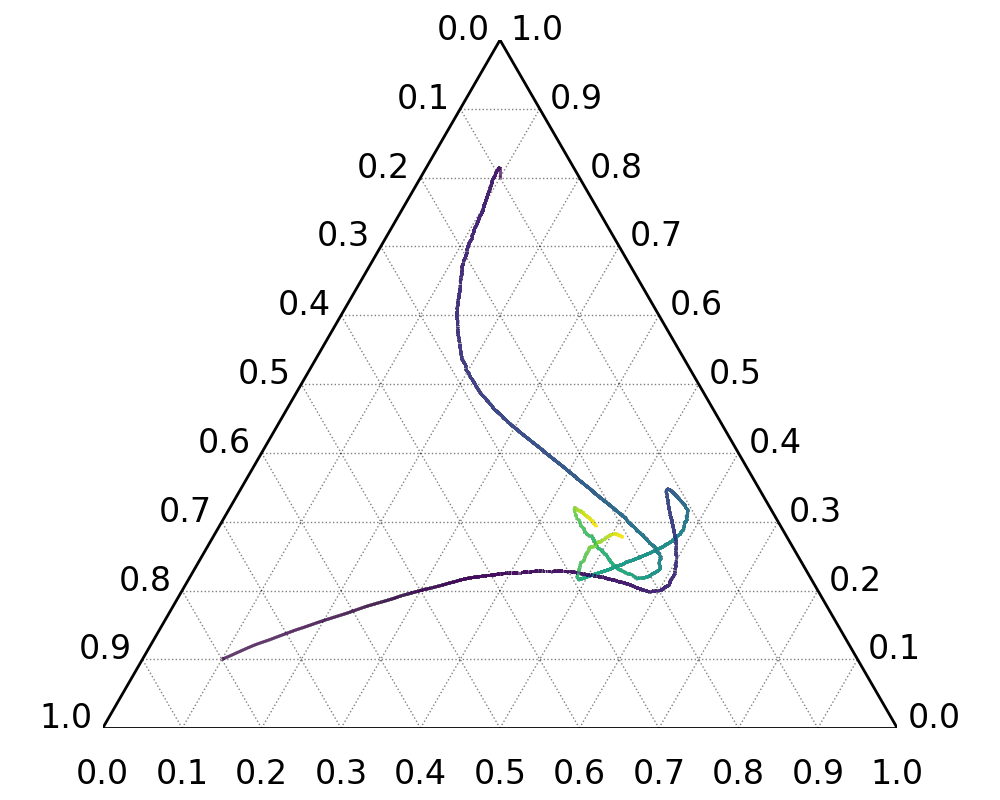}
        \label{fig:fig2b}
    \end{subfigure}
    
    \begin{subfigure}[b]{0.5\linewidth}
        \includegraphics[width=\linewidth]{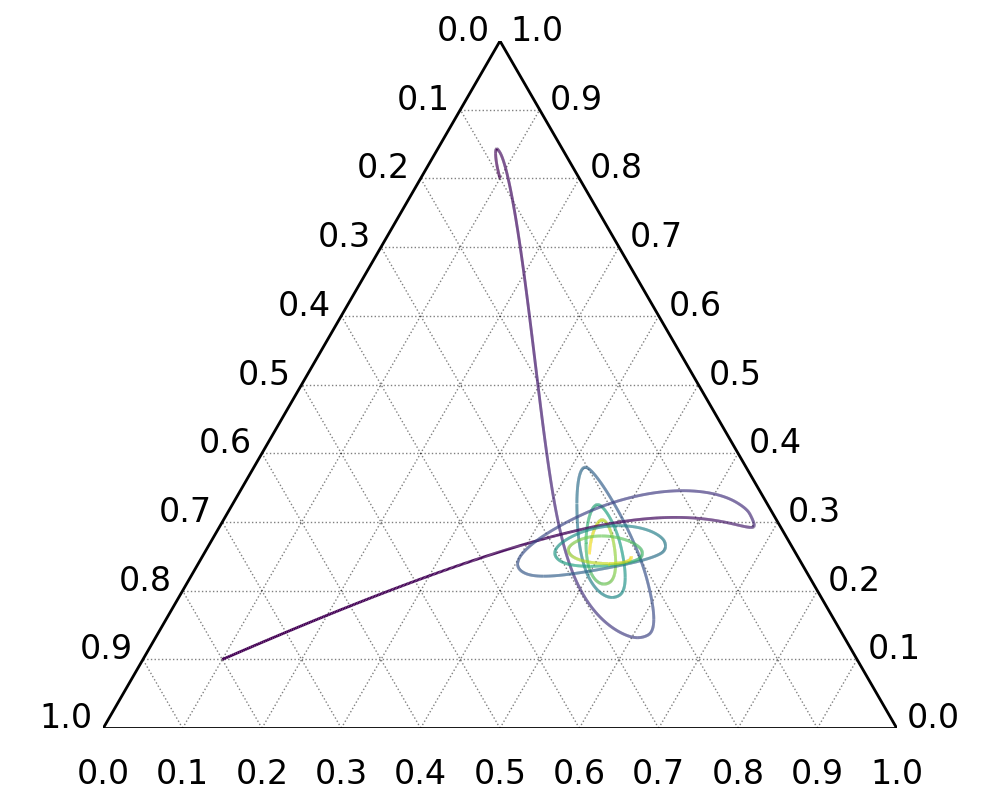}
    \end{subfigure}
    \hspace{-0.02\linewidth} 
    \begin{subfigure}[b]{0.5\linewidth}
        \includegraphics[width=\linewidth]{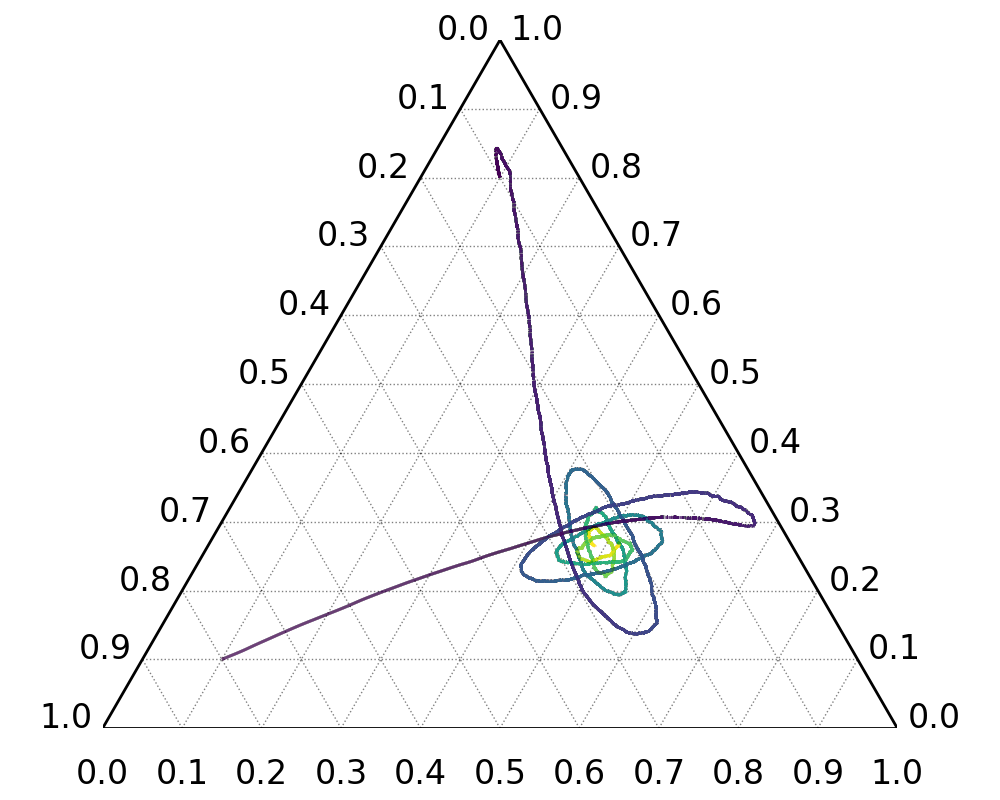}
    \end{subfigure}
    \caption{Innovative dynamics (left) vs policy trajectories (right) of the ERID algorithm using BNN dynamics (top) and Smith dynamics (bottom), respectively.}
    \label{fig:BNN_RPS_combined}
\end{figure}
As seen in Figure~\ref{fig:ex3}, before the environment begins to change, all three algorithms converge to the Nash equilibrium at a similar rate. However, as the environment begins to change, although {\sc NashConv} increases, the two agents using ERID still manage to approach the Nash equilibrium, whereas Cross Learning struggles to do so. Between steps 3e5 and 9e5, Cross Learning rapidly drifts away from the Nash equilibrium. This is due to the strategy being on a periodic orbit near the boundary of the simplex at the time of the shift. As the Nash equilibrium moves from one side of the simplex to the center, the ‘radius’ of the periodic orbit expands, forcing the strategy closer to the boundary, which hinders effective exploration of certain actions. In contrast, the ERID strategies, which converge directly to the Nash equilibrium, are not affected by this issue. By step 21e5, it is evident that Cross Learning has almost entirely lost its ability to converge to the Nash equilibrium, with the relative {\sc NashConv} remaining over 0.33—performing worse than a purely random strategy. Meanwhile, the ERID agents are unaffected by the horizon length.
\begin{figure}[t]
    \centering
    \begin{subfigure}[b]{0.8\linewidth}
        \includegraphics[width=\linewidth]{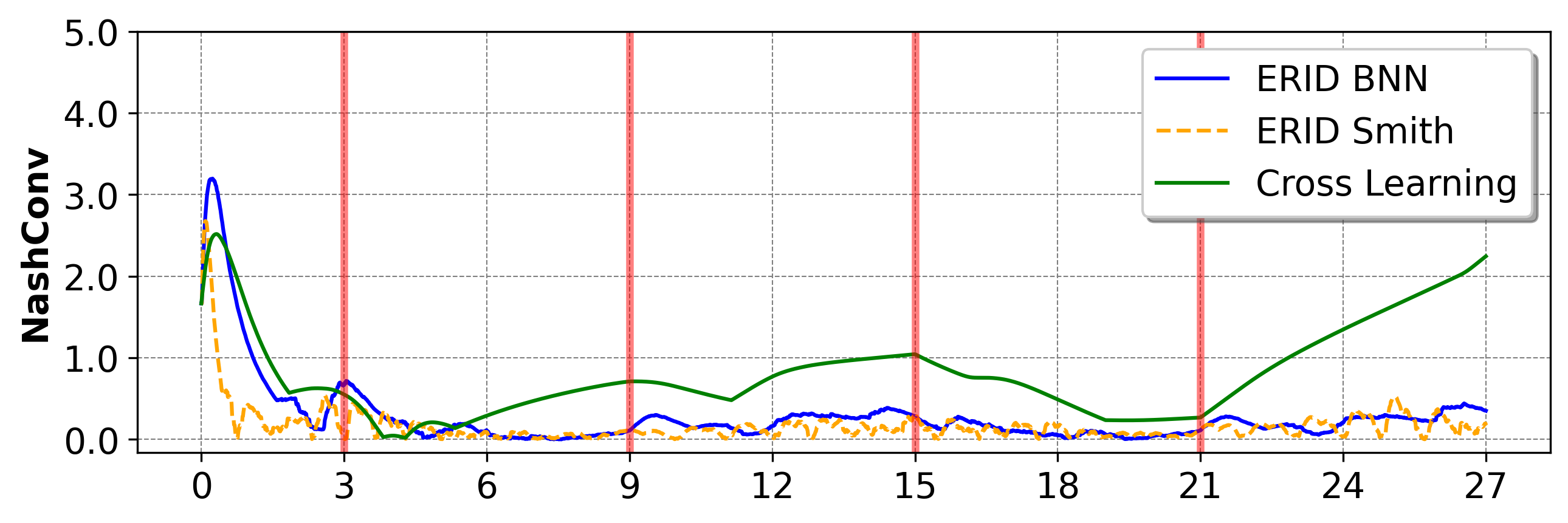}
        \label{fig:fig1a}
    \end{subfigure}
    
    \vspace{-0.08\linewidth}  
    
    \begin{subfigure}[b]{0.8\linewidth}
        \includegraphics[width=\linewidth]{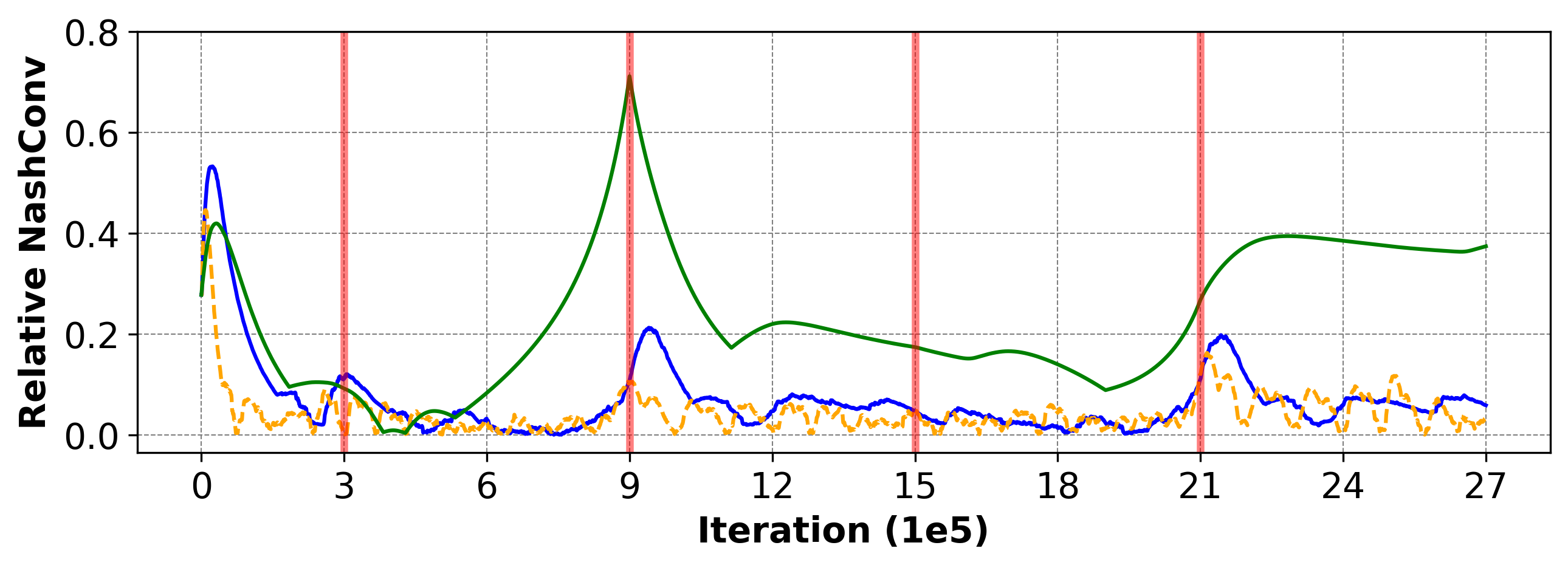}
        \label{fig:fig1b}
    \end{subfigure}
    \caption{Policy {\sc NashConv} and relative {\sc NashConv} of ERID with BNN, ERID with Smith, Cross learning in nonstationary RPS.}
    \label{fig:ex3}
\end{figure}

When comparing the two agents using ERID, ERID with BNN converges to the Nash equilibrium more slowly than ERID with Smith. However, in the long run, ERID with BNN appears to be closer to the Nash equilibrium.

\section{Conclusion}\label{sec:conc}
\ls{In this paper, we have proposed a new algorithm based on innovative dynamics, ERID, which is able to capture dynamic changes in the environment more effectively than traditional approaches based on replicator dynamics and their time-averaging counterpart. Through different revision protocols, our proposed algorithm is able to accommodate different protocol factors corresponding to three sets of dynamics, namely, BNN, Smith, and Smith-replicator-based pairwise comparison dynamics. We have demonstrated the convergence of our algorithm to the corresponding dynamics and provided convergence guarantees to the Nash equilibrium. Finally, we have evaluated the effectiveness of the proposed algorithm through a set of simulations including standard games such as matching pennies and rock-paper-scissors.}

Future \ls{directions of research include: i)} investigating the characteristics of different dynamics in learning environments, such as passivity, to more accurately determine the applicability of the algorithm, and ii) studying new dynamics through revision protocols, such as those with constraints, to develop algorithms tailored to specific environments \ls{that consider safety-critical aspects.}

\nocite{*}

\bibliographystyle{ACM-Reference-Format}
\bibliography{AAMAS_vfinal}

\end{document}